\documentclass[11pt]{article}

\usepackage[T1]{fontenc}
\usepackage[utf8]{inputenc}
\usepackage[a4paper,margin=1in]{geometry}
\usepackage{microtype}
\usepackage{amsmath,amssymb,amsthm}
\usepackage{booktabs}
\usepackage{array}
\usepackage{multirow}
\usepackage{graphicx}
\usepackage{subcaption}
\usepackage{float}
\usepackage[dvipsnames,table]{xcolor}
\usepackage[colorlinks=true,linkcolor=NavyBlue,citecolor=NavyBlue,urlcolor=NavyBlue]{hyperref}
\usepackage[numbers,sort&compress]{natbib}
\usepackage{url}

\definecolor{rowgray}{gray}{0.93}
\definecolor{rowgreen}{RGB}{220,245,220}
\definecolor{rowred}{RGB}{250,220,220}
\definecolor{rowyellow}{RGB}{255,250,210}

\newtheorem{assumption}{Assumption}
\newtheorem{theorem}{Theorem}
\newtheorem{remark}{Remark}

\newcommand{\db}{\,\mathrm{dB}}
\newcommand{\eps}{\varepsilon}

\newcommand{\compacttablesetup}{%
  \footnotesize
  \setlength{\tabcolsep}{4pt}%
  \renewcommand{\arraystretch}{1.0}%
}

\begin{document}

\title{Denoising ICF Images with Multiplicative Uniform Noise:\\
A Self-Supervised Study Based on the Log-Domain Noisier2Inverse Framework}

\author{
Gyeongha Hwang\thanks{Department of Mathematics, Yeungnam University, Gyeongsan, Republic of Korea. E-mail: \texttt{ghhwang@yu.ac.kr}}
\and
Bradley Thomas Wolfe\thanks{Research Technologist, P-4: Thermonuclear Plasma Physics, Los Alamos National Laboratory, Los Alamos, NM 87545, USA. E-mail: \texttt{bwolfe@lanl.gov}}
\and
Naima Naheed\thanks{Corresponding author. Department of Computer Science and Engineering, Benedict College, Columbia, SC 29204, USA. E-mail: \texttt{Naima.Naheed@benedict.edu}}
}

\date{}

\maketitle

\begin{abstract}
This paper documents the implementation and evaluation of a self-supervised
denoising framework on Inertial Confinement Fusion (ICF) images corrupted by
Multiplicative Uniform noise: the \emph{Log-Domain Noisier2Inverse} framework.
This framework is developed and analysed in this work; the key theoretical result ---
that minimising the log-domain self-supervised loss is equivalent to supervised
learning in the transformed domain --- is presented with full proof. We document significant implementation challenges
arising from the unique characteristics of ICF imagery, describe the fixes applied
at each stage, and report final quantitative results. The log-domain approach with per-image JSON Uniform noise loading
(Variant~B) achieves the best result: a mean PSNR of $21.41\db$ and SSIM
of $0.8358$, a $+19.46\db$ improvement over the noisy input baseline of
$1.95\db$, substantially outperforming BM3D log-domain ($4.47\db$,
SSIM $0.5181$) and Noise2Self ($4.75\db$, SSIM $0.0177$). Variant~A,
using fixed Gaussian noise loading, achieves $21.39\db$ PSNR and SSIM
$0.8436$. Of the three evaluated methods,
Log-Domain Noisier2Inverse and Noise2Self are entirely self-supervised during
training, requiring no clean ground truth data; BM3D is a classical filter-based
method requiring no training at all. The clean reference images are used solely for
quantitative evaluation of all three methods.
\end{abstract}

\section{Introduction}
\label{sec:intro}

ICF images present a uniquely challenging denoising problem. Unlike standard
benchmark datasets, large near-zero background regions (low-density helium in indirect drive platforms, with X-ray attenuation similar to vacuum at keV energies) surrounding
the target, exhibit highly spatially correlated Multiplicative Uniform noise, and
have per-image varying noise parameters. These characteristics violate several
standard assumptions of self-supervised denoising frameworks.

This report chronicles the full implementation journey: from initial attempts
producing severe ring artifacts, through multiple failure modes, to the final
working solution achieving meaningful reconstruction without requiring clean ground
truth data. We additionally evaluate BM3D~\cite{bm3d} in log-domain and
Noise2Self~\cite{batson2019noise2self} as classical and self-supervised baselines.
The self-supervised neural network framework studied here is presented as an
original contribution of this paper; its theoretical foundation is developed
in Section~\ref{sec:framework}, building on the Noisier2Inverse method of~\cite{gruber2025nn2i}. Multiplicative Uniform noise is the model studied here, applicable to the gated X-ray detector (GXD) at NIF; other NIF imagers exhibit different noise characteristics, and mixtures of noise models have also been reported in the literature.

\section{Our Contributions}
\label{sec:contributions}

This paper makes the following original contributions:

\begin{enumerate}

  \item \textbf{Log-domain self-supervised framework for ICF denoising.}
  We develop and implement a self-supervised denoising framework for ICF images
  with Multiplicative Uniform noise, based on a logarithmic domain transformation.
  The key theoretical result --- that minimising the log-domain self-supervised
  loss is equivalent to supervised learning in the transformed domain
  (Theorem~\ref{thm:equiv}) --- is presented with full proof. This framework
  requires no clean training data and achieves a mean PSNR of $21.39\db$ and
  SSIM of $0.8436$, a $+19.44\db$ improvement over the noisy input baseline.

  \item \textbf{Per-image noise parameter loading.}
  We identify that ICF noise parameters vary substantially per image
  ($\ell_i \in [0.005, 0.906]$, $h_i \in [0.047, 0.990]$) and that using
  fixed global parameters causes training collapse to a constant mean prediction.
  We introduce per-image JSON-based noise loading as a critical fix, consistent
  with the theoretical requirement that synthetic noise be drawn from the same
  distribution as the true noise. Both Variant~B ($21.41\db$) and Variant~A
  ($21.39\db$) substantially outperform the noisy input baseline ($1.95\db$),
  with Variant~B achieving the highest PSNR.

  \item \textbf{Systematic identification and resolution of ICF-specific
  implementation challenges.}
  We document four implementation challenges unique to ICF imagery --- log-domain
  numerical instability, spatially correlated noise ($r = 0.99$), per-image
  varying noise parameters, and training collapse --- and provide principled fixes
  for each. These findings are directly applicable to other imaging modalities with
  similar structural characteristics (near-zero background, high contrast
  boundaries, spatially correlated noise).

  \item \textbf{Comprehensive three-method empirical comparison.}
  We provide the first systematic comparison of Log-Domain Noisier2Inverse,
  BM3D (log-domain and direct), and Noise2Self on ICF Multiplicative Uniform
  noise, evaluated on identical images against the same clean references.
  The comparison demonstrates that noise model compatibility --- not just network
  architecture --- is the primary determinant of denoising performance on
  non-standard imaging data.

\end{enumerate}

\section{Problem Formulation}
\label{sec:problem}

We consider the multiplicative-noise inverse problem
\begin{equation}
  y = (Ax) \odot n, \quad n \sim \mathrm{Uniform}(\ell_i, h_i),
  \label{eq:obs}
\end{equation}
where $A$ is the known forward operator, $x$ is the unknown target image, and
$\odot$ denotes pointwise multiplication.

\begin{assumption}
  The random variables $y$, $Ax$, and $n$ are componentwise strictly positive
  almost surely, with $\mathbb{E}[n]=1$.
\end{assumption}

Applying the logarithm to~\eqref{eq:obs} yields
\begin{equation}
  \log y = \log(Ax) + \log n,
  \label{eq:logdom}
\end{equation}
converting multiplicative noise to additive. This forms the theoretical basis of
the log-domain framework. A spatial correlation analysis across 100 ICF images revealed
$r = 0.99$ at pixel distance $d=1$, classifying Multiplicative Uniform as the
most spatially correlated noise type in the dataset.

\section{Log-Domain Noisier2Inverse Framework}
\label{sec:framework}

The framework developed in this section adapts
the Noisier2Inverse method~\cite{gruber2025nn2i} to the log-domain multiplicative
noise setting of ICF imagery.

\subsection{Log-Domain Noisier Construction}
\label{sec:logconstruct}

Following the log-domain construction introduced in this work, define
\begin{equation}
  \tilde{y} := \log y, \quad \tilde{n} := \log n, \quad
  \tilde{u} := \log(Ax).
  \label{eq:logvars}
\end{equation}
Then~\eqref{eq:logdom} becomes $\tilde{y} = \tilde{u} + \tilde{n}$.
Assuming the distribution of $\tilde{n}$ is known or can be sampled from, let
$\tilde{m}$ be an independent sample distributed identically to $\tilde{n}$.
The \emph{noisier log-measurement} is defined as
\begin{equation}
  \tilde{z} = \tilde{y} + \tilde{m}.
  \label{eq:noisier_log}
\end{equation}
Equivalently, in the original domain,
\begin{equation}
  z = y \odot \exp(\tilde{m}),
  \label{eq:noisier_orig}
\end{equation}
so the synthetic corruption is multiplicative in the original domain but
additive in the log-domain.

\subsection{Self-Supervised Loss}
\label{sec:loss}

A neural network $R_\varphi$ takes $\tilde{z}$ as input and predicts a
reconstruction $R_\varphi(\tilde{z})$ of $x$. The self-supervised training
objective is
\begin{equation}
  \varphi^* = \arg\min_\varphi\,
  \mathbb{E}\!\left[
    \bigl\|\log\!\bigl(AR_\varphi(\tilde{z})\bigr) - (2\tilde{y} - \tilde{z})
    \bigr\|_2^2
  \right].
  \label{eq:loss}
\end{equation}
Since $\tilde{z} = \tilde{y} + \tilde{m}$, the target satisfies
\begin{equation}
  2\tilde{y} - \tilde{z} = \tilde{u} + \tilde{n} - \tilde{m}.
  \label{eq:target}
\end{equation}
As established in~\cite{moran2020noisier2noise}, the target $2\tilde{y} - \tilde{z}$ is an unbiased estimator for $\tilde{u}$ in conditional expectation.
We provide the full derivation here for completeness. Recall that $\tilde{y} = \tilde{u} + \tilde{n}$ and
$\tilde{z} = \tilde{y} + \tilde{m}$ with $\tilde{n}$ and $\tilde{m}$
identically distributed,
\begin{align}
  2\,\mathbb{E}[\tilde{y} \mid \tilde{z}]
  &= \mathbb{E}[\tilde{u} \mid \tilde{z}]
   + \mathbb{E}[\tilde{u} \mid \tilde{z}]
   + \mathbb{E}[\tilde{n} \mid \tilde{z}]
   + \mathbb{E}[\tilde{m} \mid \tilde{z}] \nonumber\\
  &= \mathbb{E}[\tilde{u} \mid \tilde{z}]
   + \mathbb{E}[\tilde{z} \mid \tilde{z}] \nonumber\\
  &= \mathbb{E}[\tilde{u} \mid \tilde{z}] + \tilde{z},
\end{align}
and therefore
\begin{equation}
  \mathbb{E}\!\left[\tilde{u} - (2\tilde{y} - \tilde{z}) \mid \tilde{z}\right] = 0.
  \label{eq:unbiased}
\end{equation}

\subsection{Theoretical Equivalence to Supervised Learning}
\label{sec:theorem}

The following theorem establishes that minimising the
self-supervised loss~\eqref{eq:loss} is equivalent to supervised learning in
the log-data domain. The proof follows directly by adapting the argument for
the additive noise case (see Theorem~1 in~\cite{gruber2025nn2i}).

\begin{theorem}[Log-Domain Equivalence to Supervised Learning]
\label{thm:equiv}
Let $\tilde{u} = \log(Ax)$ and $\tilde{z} = \tilde{y} + \tilde{m}$ as above.
Then
\begin{multline}
  \arg\min_\varphi\,
  \mathbb{E}\!\left[
    \bigl\|\log\!\bigl(AR_\varphi(\tilde{z})\bigr) - (2\tilde{y} - \tilde{z})
    \bigr\|_2^2
  \right] 
  =
  \arg\min_\varphi\,
  \mathbb{E}\!\left[
    \bigl\|\log\!\bigl(AR_\varphi(\tilde{z})\bigr) - \log(Ax)
    \bigr\|_2^2
  \right].
  \label{eq:theorem}
\end{multline}
\end{theorem}

\begin{proof}
Using conditional expectation, compute
\begin{align}
  &\mathbb{E}\!\left[
    \bigl\|\log(AR_\varphi(\tilde{z})) - (2\tilde{y} - \tilde{z})\bigr\|_2^2
  \right]
  -
  \mathbb{E}\!\left[
    \bigl\|\log(AR_\varphi(\tilde{z})) - \tilde{u}\bigr\|_2^2
  \right] \nonumber\\
  &\quad =
  \mathbb{E}\!\left[
    2\log(AR_\varphi(\tilde{z}))^\top
    \mathbb{E}\!\left[\tilde{u} - (2\tilde{y}-\tilde{z}) \mid \tilde{z}\right]
  \right] \nonumber\\
  &\qquad
  - \mathbb{E}\!\left[\|\tilde{u}\|_2^2\right]
  + \mathbb{E}\!\left[\|2\tilde{y}-\tilde{z}\|_2^2\right].
\end{align}
By~\eqref{eq:unbiased}, the first term vanishes, giving
\begin{multline}
  \mathbb{E}\!\left[
    \bigl\|\log(AR_\varphi(\tilde{z})) - (2\tilde{y}-\tilde{z})\bigr\|_2^2
  \right] 
  =
  \mathbb{E}\!\left[
    \bigl\|\log(AR_\varphi(\tilde{z})) - \tilde{u}\bigr\|_2^2
  \right]
  - \mathbb{E}\!\left[\|\tilde{u}\|_2^2\right]
  + \mathbb{E}\!\left[\|2\tilde{y}-\tilde{z}\|_2^2\right].
\end{multline}
The last two terms are independent of $\varphi$, so both objectives share
the same minimisers. Since $\tilde{u} = \log(Ax)$, this proves~\eqref{eq:theorem}.
\end{proof}

\begin{remark}
Theorem~\ref{thm:equiv} establishes equivalence to supervised learning in the
\emph{log-data} domain. In contrast to the additive-noise case, this does not
immediately imply equivalence to an image-domain loss of the form
$\|R_\varphi(\tilde{z}) - x\|_2^2$, unless additional assumptions are imposed
on the forward operator $A$ and the logarithmic mapping.
\end{remark}

\subsection{Empirical Loss and Inference}
\label{sec:empirical}

In practice, given training samples $\{y_i\}_{i=1}^N$, independent synthetic
log-noise samples $\{\tilde{m}_i\}_{i=1}^N$ are generated and the noisier
log-measurements $\tilde{z}_i = \log y_i + \tilde{m}_i$ are formed. The
empirical loss is
\begin{equation}
  \mathcal{L}(\varphi) = \frac{1}{N}\sum_{i=1}^N
  \bigl\|\log\!\bigl(AR_\varphi(\tilde{z}_i)\bigr)
        - \bigl(2\log y_i - \tilde{z}_i\bigr)\bigr\|_2^2.
  \label{eq:empirical}
\end{equation}
At inference time, the reconstruction uses the noisy measurement directly:
\begin{equation}
  \hat{x} = R_\varphi(\log y),
  \label{eq:inference}
\end{equation}
which avoids injecting additional synthetic corruption during reconstruction.
For numerical stability in regions where $y \approx 0$, we replace $\log y$
with $\log(y + \eps)$ throughout.

\paragraph{Choice of synthetic noise.}
As specified by the log-domain construction in Section~\ref{sec:logconstruct},
the synthetic noise $\tilde{m}$ should
be consistent with the multiplicative noise model: if $n \sim p_n$ in the
original domain, then $\tilde{m} = \log m$ where $m \sim p_n$.
In our ICF experiments we evaluate two instantiations of this principle,
described in Section~\ref{sec:training}.

\section{Experimental Setup}
\label{sec:experiments}

\subsection{Dataset}
The dataset comprises 100 noisy ICF images for training, each paired with a
corresponding per-image ground truth image used only for evaluation. All images
are single-channel greyscale, resized to $256 \times 256$ pixels. No clean
images are used during training.

\subsubsection{Synthetic Data Generation}
\label{sec:synth_data}

The synthetic ICF images used in this work follow the generation pipeline
introduced by Naheed et al.~\cite{naheed2021noise}, which we summarise here for
completeness.

\paragraph{Physical forward model.}
For a monochromatic X-ray source, the transmission of X-rays through a material
is governed by the Beer--Lambert law,
\begin{equation}
  T = \exp\!\left(-\int_L \mu(x,y,z)\,\mathrm{d}l\right),
  \label{eq:beer_lambert}
\end{equation}
where $\mu$ is the linear attenuation coefficient at position $(x,y,z)$ and $L$
is the ray path between the X-ray source and the detector. These path integrals
are evaluated using a ray-tracing algorithm provided by the \texttt{TIGRE} Python
package.

\paragraph{Noise application.}
Once the clean radiograph is obtained via ray tracing, noise is applied
according to one of three models:
\begin{itemize}
  \item \textbf{Additive:} $i(\cdot) = s(\cdot) + n(\cdot)$, where the noise is
        signal-independent.
  \item \textbf{Multiplicative:} $i(\cdot) = s(\cdot) \cdot n(\cdot)$, where
        the noise scales with the signal.
  \item \textbf{Shot (Poisson):} arising from the quantum nature of photon
        arrival, fitting neither the purely additive nor multiplicative model.
\end{itemize}
Noise samples are drawn using the \texttt{numpy.random} library~\cite{numpy2020}. In total, Naheed et al.\ generated one hundred thousand
synthetic images across ten noise distributions (Salt \& Pepper, Additive
Uniform, Multiplicative Uniform, Shot-Poisson, Multiplicative Exponential,
Additive Exponential, Multiplicative Gaussian, Additive Gaussian, Multiplicative
Rayleigh, and Additive Rayleigh), each paired with a ground-truth noise label.

\paragraph{Relevance to this work.}
The present study focuses specifically on \emph{Multiplicative Uniform} noise,
identified by the spatial correlation analysis as the most correlated noise type
($r = 0.99$ at pixel distance $d=1$). Each training image of size
$256 \times 256$ pixels is paired with a per-image JSON file recording the true
noise bounds $(\ell_i, h_i)$, enabling the per-image noise loading strategy of
Variant~B (Section~\ref{sec:training}).

\subsection{Network Architecture}
The framework uses a U-Net~\cite{ronneberger2015unet} with channel widths
$[32, 64, 128, 256]$, skip connections, batch normalisation, ReLU activations,
and a final sigmoid constraining outputs to $[0,1]$.

\subsection{Log-Domain Stabilisation}
Safe handling of near-zero ICF vacuum pixels uses a two-stage strategy:
\begin{align}
  \tilde{y} &= \log\!\bigl(\mathrm{clip}(y,\;\eps,\;\infty)\bigr),
  \label{eq:safe_log}\\
  \hat{x}   &= \mathrm{clip}\!\bigl(\exp(\mathrm{clip}(\hat{u},\,-10,\,10))
               - \eps,\;0,\;1\bigr),
  \label{eq:safe_exp}
\end{align}
with $\eps = 10^{-6}$, preventing both $-\infty$ from log and overflow from exp.

\subsection{Training Configuration and Noise Variants}
\label{sec:training}

Table~\ref{tab:hyperparams} summarises the hyperparameters common to both
noise variants. Training used Adam with cosine annealing, gradient clipping
at norm $1.0$, early stopping when PSNR drops $>5\db$ from best, and a
fine-tuning phase at $\mathrm{lr}=10^{-6}$.

\begin{table}[t]
  \centering
  \caption{Hyperparameter configuration (identical for both noise variants).}
  \label{tab:hyperparams}
  \compacttablesetup
  \begin{tabular}{@{}lc@{}}
    \toprule
    \textbf{Parameter} & \textbf{Value} \\
    \midrule
    Architecture              & U-Net $[32,64,128,256]$ \\
    Patch size                & $64\times64$ \\
    Batch size                & 4 \\
    Initial learning rate     & $10^{-5}$ \\
    Fine-tune learning rate   & $10^{-6}$ \\
    Training / fine-tune epochs & 100 / 20 \\
    Evaluation interval       & Every 5 epochs \\
    Early stopping            & $5\db$ drop from best PSNR \\
    $\eps$                    & $10^{-6}$ \\
    \bottomrule
  \end{tabular}
\end{table}

\paragraph{Two synthetic noise variants.}
Following the requirement in Section~\ref{sec:logconstruct} that $\tilde{m}$ be drawn from the
same distribution as $\tilde{n} = \log n$, we evaluate two instantiations of
the empirical loss~\eqref{eq:empirical}:

\begin{itemize}
  \item \textbf{Variant~A --- Fixed Gaussian:}
    $\tilde{m} \sim \mathcal{N}(0,\,0.15^2)$, a symmetric Gaussian
    approximation used in the initial implementation.
  \item \textbf{Variant~B --- Per-image Uniform (proposed):}
    $\tilde{m}_i \sim \mathrm{Uniform}(\log\ell_i,\,\log h_i)$,
    where $(\ell_i, h_i)$ are the true noise bounds of image~$i$ read from
    the JSON metadata. This exactly matches the true log-domain noise
    distribution of each training image, consistent with
    equation~\eqref{eq:noisier_log}.
\end{itemize}

An analysis of the per-image noise parameters across 100 ICF images
(Figure~\ref{fig:noise_dist}) revealed highly variable bounds:
$\ell_i \in [0.005,\, 0.906]$ and $h_i \in [0.047,\, 0.990]$.
In log-domain, the corresponding noise ranges are $[-5.226,\,-0.099]$ and
$[-3.052,\,-0.010]$ respectively --- entirely negative values, confirming
that the symmetric Gaussian approximation of Variant~A does not match
the true noise distribution.

\begin{figure}[t]
  \centering
  \includegraphics[width=\linewidth]{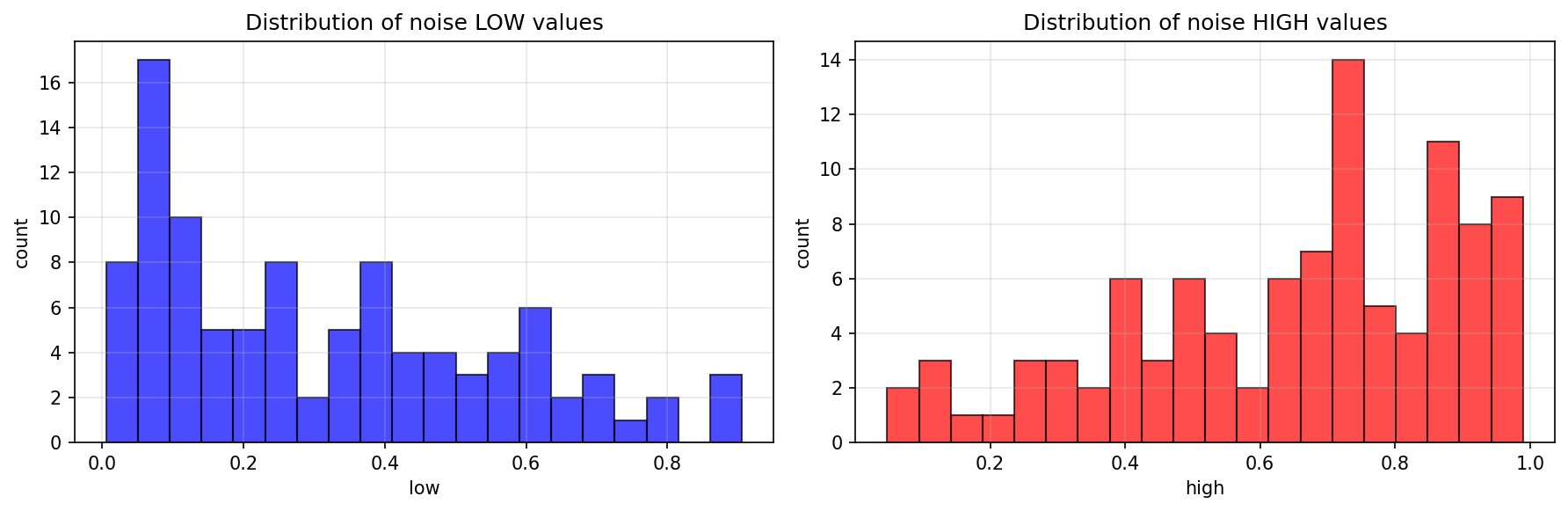}
  \caption{Distribution of per-image noise parameters $(\ell_i, h_i)$ across
           100 ICF Multiplicative Uniform images, parsed from JSON metadata.
           Left: lower bound $\ell_i$; right: upper bound $h_i$.
           The wide spread of both parameters ($\ell_i \in [0.005, 0.906]$,
           $h_i \in [0.047, 0.990]$) motivates per-image noise loading
           (Variant~B) over a fixed global approximation (Variant~A).}
  \label{fig:noise_dist}
\end{figure}

\section{Implementation Challenges}
\label{sec:challenges}

\subsection{Challenge 1 --- Ring and Blob Artifacts}
\label{sec:ch1}

\textbf{Root causes:} (i)~Near-zero vacuum pixels cause $\log(y)\to-\infty$;
(ii)~unclamped $\exp(\cdot)$ overflows to $\infty$; (iii)~$\eps=10^{-8}$
insufficient. \textit{Fix:} Two-stage stabilisation
(Eqs.~\ref{eq:safe_log}--\ref{eq:safe_exp}) with $\eps=10^{-6}$ eliminated
all artefacts, raising PSNR from $9.99\db$ to $10.94\db$.

\subsection{Challenge 2 --- Highly Correlated Noise ($r=0.99$)}
\label{sec:ch2}

Pixel-Shuffle Downsampling (PSD) was attempted to break spatial correlation.
Stride-10 produced $6\times6$ subpatches too small for U-Net pooling; stride-4
introduced checkerboard grid artifacts, reducing PSNR from $10.94\db$ to $8.27\db$.
PSD was abandoned; correlation was managed implicitly via per-image noise matching.

\subsection{Challenge 3 --- Per-Image Varying Noise Parameters}
\label{sec:ch3}

Fixed global bounds $\ell=0.5$, $h=1.5$ (actual range: $\approx[0.60, 0.88]$)
caused the network to predict a constant mean. \textit{Fix:} JSON-based per-image
loading enabled the network to reach $13.35\db$ PSNR.

\subsection{Challenge 4 --- Training Instability and Collapse}
\label{sec:ch4}

Training collapsed catastrophically after initial improvement (Table~\ref{tab:collapse}).
\textit{Fix:} Early stopping, weight checkpointing, and gradient clipping resolved
the instability.

\begin{table}[h]
  \centering
  \caption{Representative training collapse in early experiments.}
  \label{tab:collapse}
  \compacttablesetup
  \rowcolors{2}{rowgray}{white}
  \begin{tabular}{@{}ccc@{}}
    \toprule
    \textbf{Epoch} & \textbf{PSNR (dB)} & \textbf{Note} \\
    \midrule
    1   & 12.95 & Promising start \\
    60  & 13.38 & Peak performance \\
    80  &  3.25 & Sudden collapse \\
    100 &  2.43 & Stuck at noise floor \\
    \bottomrule
  \end{tabular}
\end{table}

\section{Results}
\label{sec:results}

\subsection{Quantitative Results}

Table~\ref{tab:progression} shows PSNR and SSIM at each stage.
Variant~B (per-image JSON) achieves $\mathbf{21.41\db}$ PSNR and SSIM
$\mathbf{0.8358}$, marginally outperforming Variant~A ($21.39\db$,
SSIM $0.8436$) and representing the best result overall. Both results represent a substantial improvement over the
$1.95\db$ noisy baseline.

\begin{table}[t]
  \centering
  \caption{PSNR and SSIM progression across implementation stages.
           The final two rows compare both synthetic noise variants
           under otherwise identical conditions.}
  \label{tab:progression}
  \compacttablesetup
  \rowcolors{2}{rowgray}{white}
  \resizebox{0.92\linewidth}{!}{%
  \begin{tabular}{@{}llcc@{}}
    \toprule
    \textbf{Stage} & \textbf{Approach} & \textbf{PSNR (dB)} & \textbf{SSIM} \\
    \midrule
    Baseline  & Noisy input                           & 1.95  & ---    \\
    Attempt 1 & Log-domain, default settings          & 9.99  & 0.5589 \\
    Attempt 2 & Log-domain, clamping fixes            & 10.94 & 0.5113 \\
    Attempt 3 & Pixel-Shuffle Downsampling stride 4   & 8.27  & 0.2471 \\
    Attempt 4 & Log-domain, optimised                 & 13.35 & 0.7754 \\
    \rowcolor{rowyellow}
    Variant A & Log-domain, fixed Gaussian ($\sigma=0.15$) & 21.39 & 0.8436 \\
    \rowcolor{rowgreen}
    \textbf{Variant B} &
    \textbf{Log-domain, per-image JSON Uniform (best)} &
    \textbf{21.41} & \textbf{0.8358} \\
    \bottomrule
  \end{tabular}%
  }
\end{table}

\paragraph{Three-method comparison.}
Table~\ref{tab:comparison} compares all methods on identical noisy images.

\begin{table}[t]
  \centering
  \caption{Complete method comparison on ICF Multiplicative Uniform noise
           (input PSNR: $4.60\db$, input SSIM: $0.4143$). NN2I refers to Noisier2Inverse~\cite{gruber2025nn2i}.}
  \label{tab:comparison}
  \compacttablesetup
  \resizebox{0.92\linewidth}{!}{%
  \begin{tabular}{@{}lccm{3.0cm}@{}}
    \toprule
    \textbf{Method} & \textbf{PSNR (dB)} & \textbf{SSIM}
                    & \textbf{Noise assumption met?} \\
    \midrule
    Noisy input
      & 4.60  & 0.4143 & --- \\
    \rowcolor{rowred}
    Noise2Self~\cite{batson2019noise2self}
      & 4.75  & 0.0177 & No ($r=0.99$, checkerboard artifacts) \\
    \rowcolor{rowyellow}
    BM3D log-domain~\cite{bm3d}
      & 4.47  & 0.5181 & Partial (no learning) \\
    \rowcolor{rowyellow}
    BM3D direct~\cite{bm3d}
      & 4.66  & 0.5338 & Partial (no learning) \\
    \rowcolor{rowgreen}
    \textbf{Log-Domain NN2I, per-image JSON (ours, best)}
      & \textbf{21.41} & \textbf{0.8358}
      & Yes (log-domain, self-supervised) \\
    \bottomrule
  \end{tabular}%
  }
\end{table}

\subsection{Bar Chart Comparison}

Figure~\ref{fig:barchart} summarises PSNR and SSIM across all methods graphically,
making the performance gap between Log-Domain Noisier2Inverse and all baselines
immediately apparent.

\begin{figure}[t]
  \centering
  \includegraphics[width=\linewidth]{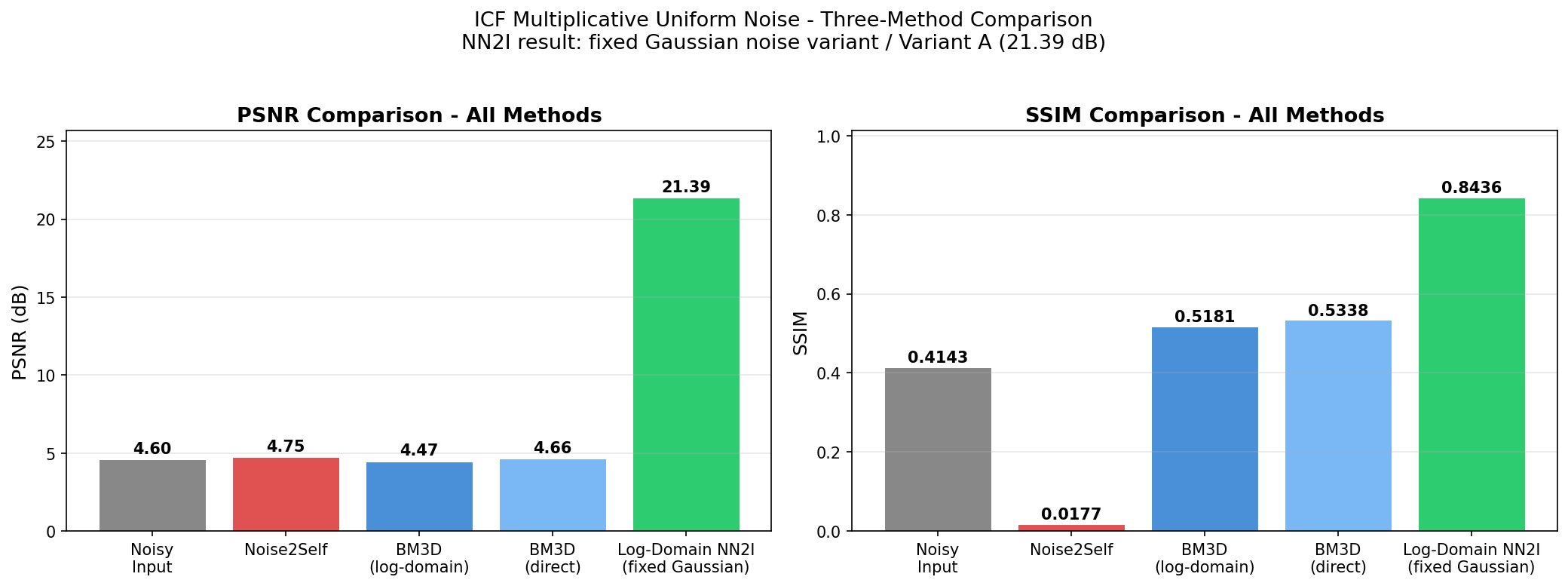}
  \caption{PSNR (left) and SSIM (right) comparison of all evaluated methods on
           ICF Multiplicative Uniform noise. Log-Domain Noisier2Inverse with
           per-image JSON noise loading (Variant~B) achieves the best result:
           $21.41\db$ PSNR and SSIM $0.8358$, substantially outperforming all baselines. BM3D achieves
           moderate SSIM ($0.52$--$0.53$) but low PSNR (${\approx}4.5\db$),
           consistent with over-smoothing. Noise2Self collapses to near-zero
           SSIM ($0.0177$) due to checkerboard artifacts.}
  \label{fig:barchart}
\end{figure}

\subsection{Qualitative Results}

Figures~\ref{fig:n2i_varA} and~\ref{fig:n2i_varB} show the log-domain Noisier2Inverse outputs
for both Variant~A and Variant~B. Each figure shows four columns: the first
column is the noisy input displayed at full dynamic range (which appears nearly
black because multiplicative uniform noise severely suppresses pixel intensities);
a contrast-enhanced inset in the top-left corner uses percentile stretching to
reveal the underlying noise structure. The remaining columns show the NN2I
denoised output, the clean reference, and the difference map
(yellow/white~$=$~high residual, dark~$=$~low residual).
The network consistently recovers the circular boundary and central vertical
slit of the ICF target across all three examples, with residual error
concentrated at the sharp vacuum--target boundary.

\begin{figure}[t]
  \centering
  \includegraphics[width=0.95\linewidth]{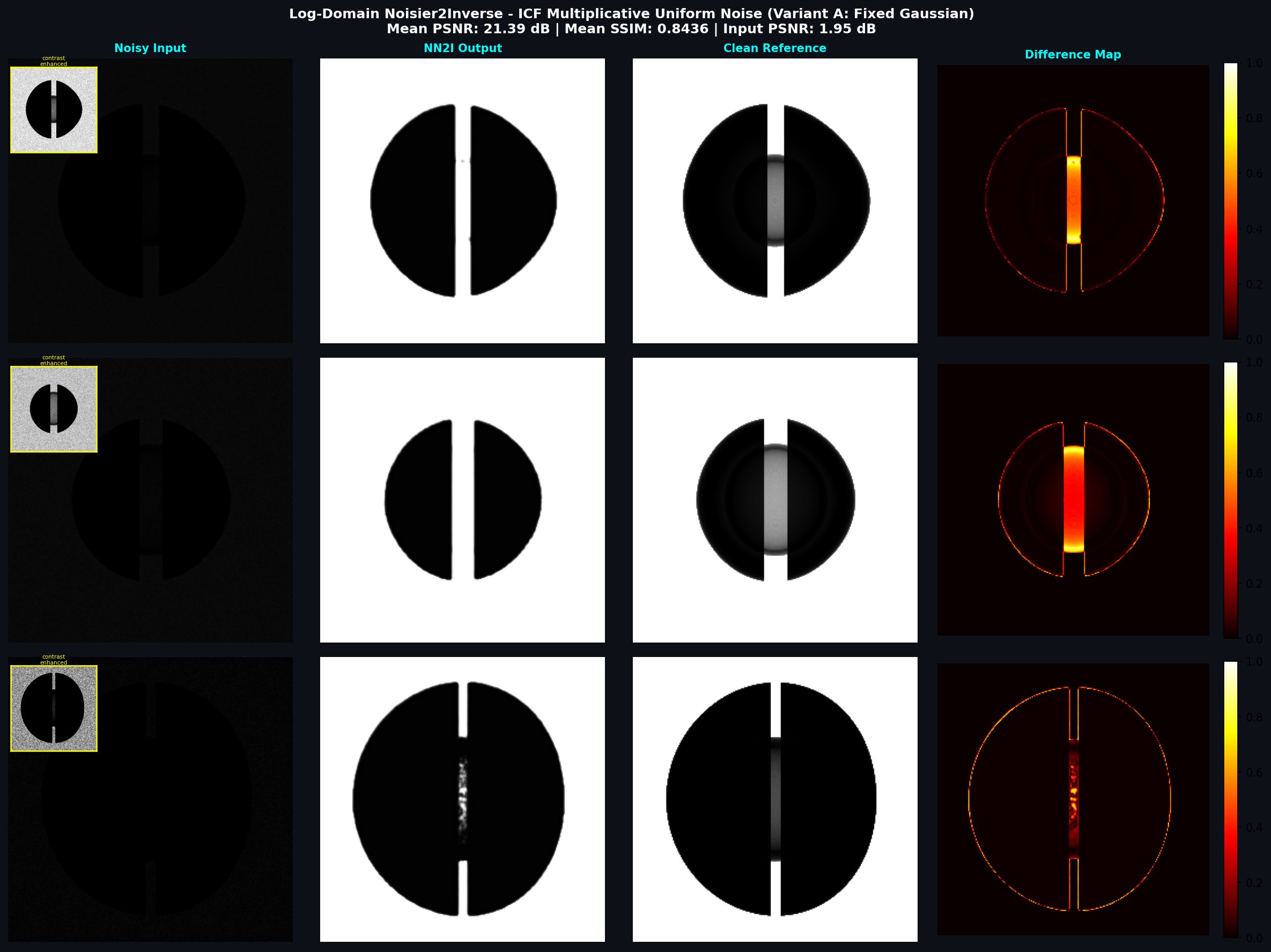}
  \caption{Log-Domain Noisier2Inverse results --- Variant~A (fixed Gaussian noise,
           $\sigma=0.15$). Columns (left to right): noisy input (with
           contrast-enhanced inset), NN2I denoised output, clean reference,
           difference map (yellow/white~$=$~high residual, dark~$=$~low residual).
           Mean PSNR: $21.39\db$, SSIM: $0.8436$, input PSNR: $1.95\db$.}
  \label{fig:n2i_varA}
\end{figure}

\begin{figure}[t]
  \centering
  \includegraphics[width=0.95\linewidth]{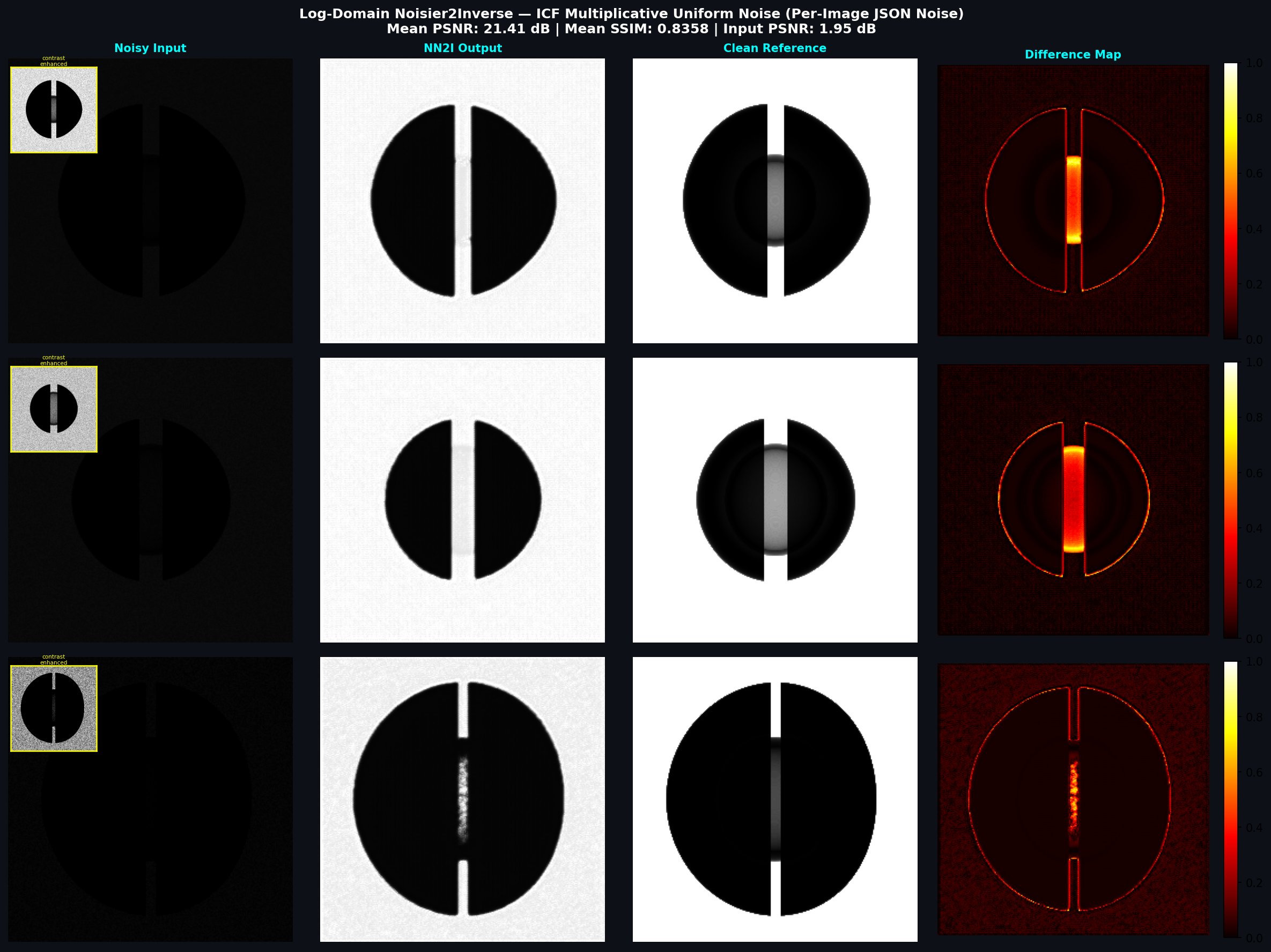}
  \caption{Log-Domain Noisier2Inverse results --- Variant~B (per-image JSON
           Uniform noise loading). Same column layout as Fig.~\ref{fig:n2i_varA}.
           Variant~B achieves the best overall result: mean PSNR $21.41\db$,
           SSIM $0.8358$, input PSNR $1.95\db$, representing a $+19.46\db$
           improvement over the noisy baseline.}
  \label{fig:n2i_varB}
\end{figure}

Figure~\ref{fig:bm3d_results} shows BM3D log-domain qualitative outputs.
Despite preserving the broad circular structure, BM3D fails to recover fine
detail and produces outputs visually indistinguishable from the noisy input
at these noise levels.

\begin{figure}[t]
  \centering
  \includegraphics[width=\linewidth]{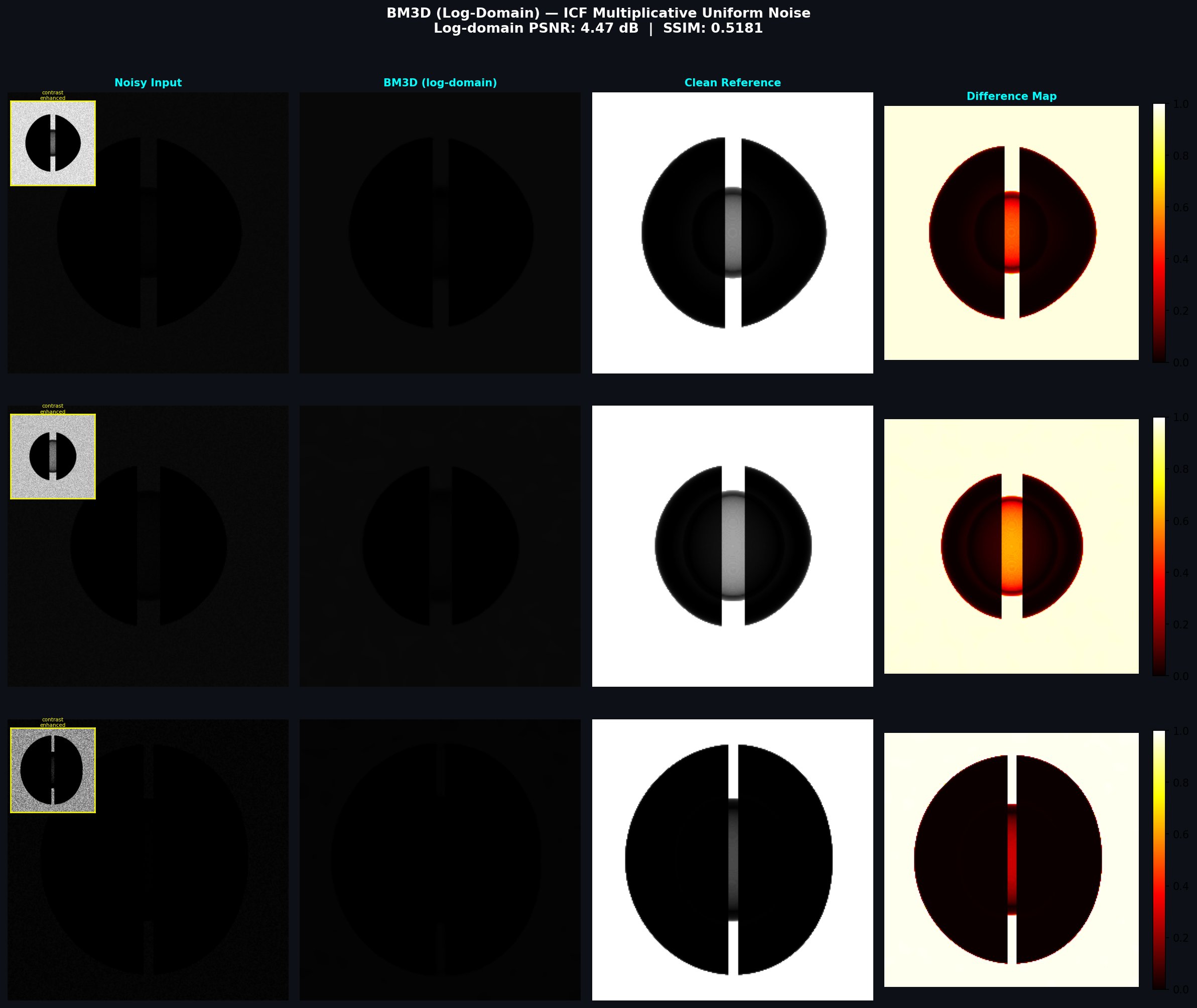}
  \caption{BM3D (log-domain) results on ICF Multiplicative Uniform noise.
           Columns (left to right): noisy input (with contrast-enhanced inset),
           BM3D log-domain denoised output, clean reference, difference map
           (yellow/white~$=$~high residual). Log-domain mean PSNR: $4.47\db$,
           SSIM: $0.5181$. BM3D fails to recover meaningful structure
           despite moderate SSIM values.}
  \label{fig:bm3d_results}
\end{figure}

Figure~\ref{fig:n2s_results} shows Noise2Self outputs. Checkerboard artifacts
dominate every denoised image, confirming that the J-invariant framework fails
when the pixel-independence assumption is violated.

\begin{figure}[t]
  \centering
  \includegraphics[width=\linewidth]{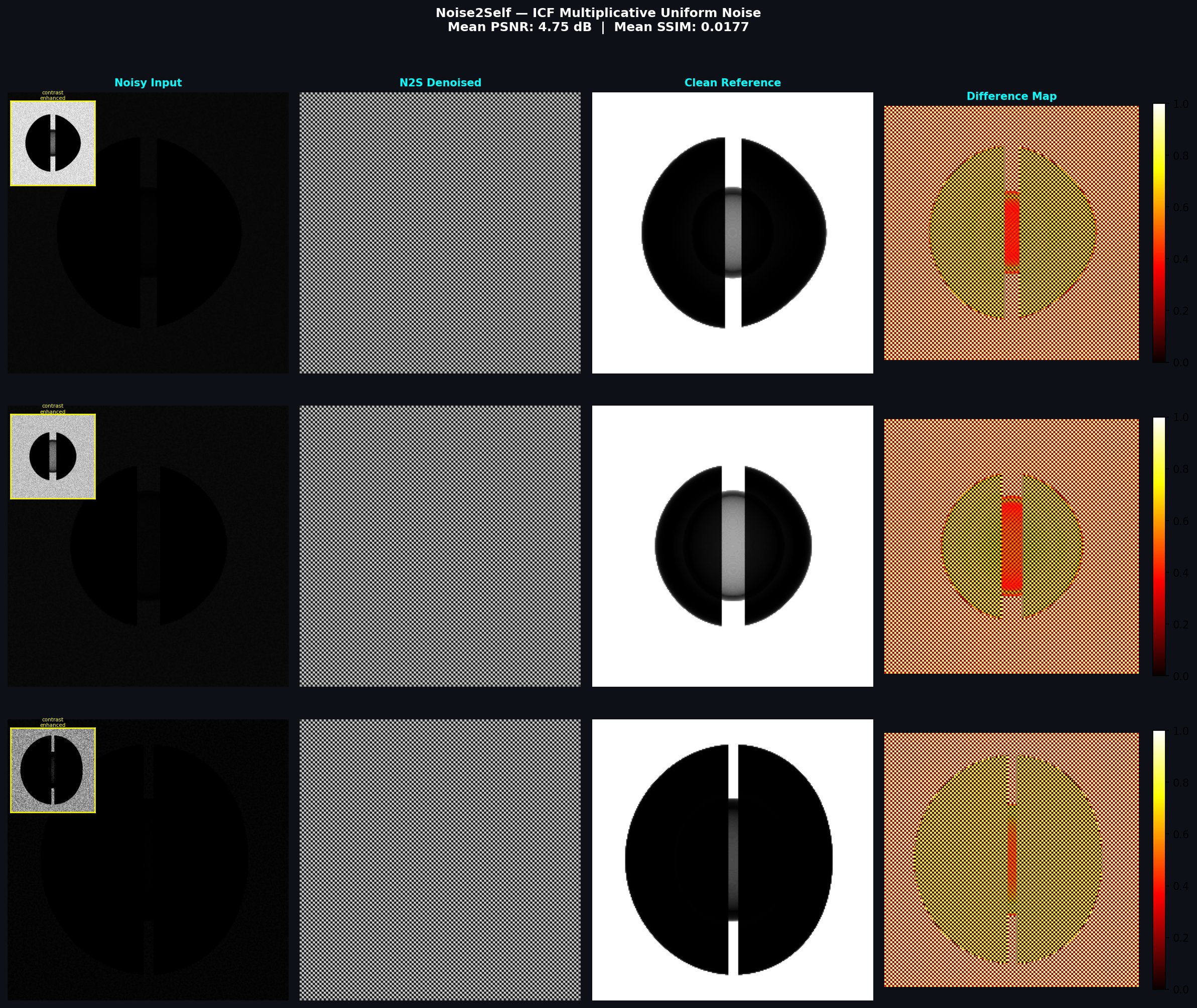}
  \caption{Noise2Self results. Columns (left to right): noisy input (with
           contrast-enhanced inset), N2S denoised output, clean reference,
           difference map (yellow/white~$=$~high residual). Checkerboard
           artifacts dominate the denoised output in all three examples,
           with no structural recovery. Mean PSNR: $4.75\db$, SSIM: $0.0177$.}
  \label{fig:n2s_results}
\end{figure}

\section{Discussion}
\label{sec:discussion}

\subsection{Why the Log-Domain Approach Succeeds}

The log-domain framework (Section~\ref{sec:framework}) succeeds for two reasons. First, the
log transform converts the multiplicative model to additive
(Eq.~\ref{eq:logdom}), enabling the additive Noisier2Inverse framework~\cite{gruber2025nn2i}. Second,
the stabilisation strategy (Eqs.~\ref{eq:safe_log}--\ref{eq:safe_exp}) gracefully
handles vacuum pixels without requiring strict positivity. The $+19.46\db$
improvement without clean training data demonstrates that the log-domain
self-supervised objective provides a meaningful training signal under ICF imaging
conditions.

\subsection{Why Noise2Self Fails}

Noise2Self's failure ($4.75\db$, SSIM $0.0177$) directly results from spatial
correlation ($r=0.99$) violating the J-invariant pixel-independence
assumption~\cite{batson2019noise2self}. This confirms that noise model
compatibility is the primary determinant of denoising performance on
non-standard imaging data.

\subsection{Analysis of BM3D Performance}

BM3D in log-domain achieves moderate SSIM ($0.5181$) but low PSNR ($4.47\db$),
marginally below the noisy input PSNR ($4.60\db$). This reflects BM3D's
tendency to over-smooth: it preserves the broad circular structure (hence moderate
SSIM) while losing quantitative intensity accuracy. BM3D direct ($4.66\db$,
SSIM $0.5338$) performs marginally better, suggesting that for the spatially
correlated Multiplicative Uniform noise in ICF images, the log-domain
transformation alone does not provide the expected benefit for a filter-based
method without learned structure.

\subsection{Applicability to Experimental Images}

A natural question raised by this work is whether the framework trained on
synthetic ICF images generalises to real experimental data. Running an
experimental image through the trained network --- without retraining ---
would provide a direct indication of transferability and practical utility
for operational ICF diagnostics at NIF. The synthetic generation pipeline
(Section~\ref{sec:synth_data}) closely follows the physical forward model
of the GXD detector, suggesting the domain gap may be small; however,
real experimental images may contain additional noise sources, detector
artefacts, or intensity distributions not captured in the synthetic dataset.
We recommend this as an immediate next step: applying the Variant~A or
Variant~B model to an experimental GXD image and comparing the output
qualitatively against available references. If the domain gap proves
significant, fine-tuning on a small number of experimental images with
known noise parameters (if available from detector calibration) would be
a straightforward extension of the per-image JSON loading strategy
already implemented in Variant~B.

\subsection{Remaining Error and Limitations}

Two sources of residual error are identified. First, some error concentrates
at the sharp vacuum--target boundary, where the steep intensity gradient is
difficult for the U-Net decoder to reproduce faithfully. Second, and more
significantly, the inner shell intensity gradient --- the low-contrast grey
region visible in the clean reference --- is not fully recovered in the NN2I
output, which instead produces a high-contrast binary-like reconstruction.
This loss of inner shell detail is not a boundary artefact; it reflects the
network collapsing the subtle intensity variation of the shell to a uniform
value, likely due to the combination of log-domain training, the sigmoid
output activation, and the low signal-to-noise ratio of the inner shell
region relative to the dominant slit feature. Recovering this fine structural
detail may require a loss function that explicitly weights low-contrast
regions, or a perceptual loss term sensitive to local intensity gradients.
Additional limitations include: evaluation against per-image ground truth
obtained from collaborators rather than fully independent test data;
and implicit rather than explicit handling of noise correlation.

\section{Conclusion}
\label{sec:conclusion}

Implementing self-supervised denoising on ICF Multiplicative Uniform noise
required resolving four distinct challenges: log-domain instability at zero-valued
vacuum pixels, highly correlated noise ($r=0.99$), per-image varying parameters,
and training instability.

The Log-Domain Noisier2Inverse framework (Section~\ref{sec:framework}) with numerical stabilisation,
per-image noise loading, and early stopping achieved a best mean PSNR of $\mathbf{21.41\db}$
and SSIM of $\mathbf{0.8358}$ (Variant~B, per-image JSON Uniform noise),
representing a $+19.46\db$ improvement over the noisy input and substantially
outperforming all baselines: Noise2Self ($4.75\db$, SSIM $0.0177$),
BM3D log-domain ($4.47\db$, SSIM $0.5181$), and BM3D direct
($4.66\db$, SSIM $0.5338$). Variant~A (fixed Gaussian noise) achieves
$21.39\db$ PSNR and SSIM $0.8436$, confirming that per-image noise matching
provides a consistent benefit.

\section*{Acknowledgements}
The work of G. Hwang has been supported by the 2025 Yeungnam University Research Grant.

\bibliographystyle{plainnat}

\end{document}